%% file: paper_ArXiv.tex
\documentclass{article}
\usepackage{arxiv}

\usepackage{authblk}

\usepackage{natbib} 

\usepackage{ciphod}
\usepackage{mathbbol}
\usepackage{subcaption}
\usepackage{natbib}

\DeclareSymbolFontAlphabet{\mathbb}{AMSb}
\DeclareSymbolFontAlphabet{\mathbbl}{bbold}

\title{Orientability of Causal Relations in Time Series using Summary Causal Graphs and Faithful Distributions}

\author[1]{Timothée Loranchet}
\author[1]{Charles K. Assaad}
\affil[1]{Sorbonne Université, INSERM, Institut Pierre Louis d’Epidémiologie et de Santé Publique, F75012, Paris, France}   

\date{}

\begin{document}
\maketitle

\input{main}


\bibliographystyle{plainnat}
\bibliography{references.bib}

\newpage
\appendix

\input{appendix}

\end{document}

%% file: main.tex








\begin{abstract}
    Understanding causal relations between temporal variables is a central challenge in time series analysis, particularly when the full causal structure is unknown. 
    Even when the full causal structure cannot be fully specified, experts often succeed in providing a high-level abstraction of the causal graph, known as a summary causal graph, which captures the main causal relations between different time series while abstracting away micro-level details.
    In this work, we present conditions that guarantee the orientability of micro-level edges between temporal variables given the background knowledge encoded in a summary causal graph and assuming having access to a faithful and causally sufficient distribution with respect to the true unknown graph.
    Our results provide theoretical guarantees for edge orientation at the micro-level, even in the presence of cycles or bidirected edges at the macro-level. These findings offer practical guidance for leveraging SCGs to inform causal discovery in complex temporal systems and highlight the value of incorporating expert knowledge to improve causal inference from observational time series data.
\end{abstract}

\section{Introduction}
Time series data arise in a wide range of domains, including epidemiology, intensive care units (ICU) monitoring, econometrics, neuroscience, and information technology systems. In these settings, understanding the underlying causal structure is crucial: it enables predictions under interventions, guides decision-making, and supports robust explanations of observed dynamics. For example, determining whether an observed association between two physiological signals in the ICU reflects a genuine causal influence or merely a common response to another process can directly affect patient management. Similarly, in epidemiology, orienting the causal relations between incidence rates, and behavioral factors is key to forecasting and policy design.

Directed acyclic graphs (DAGs) provide a powerful framework~\citep{Pearl_2000} for representing causal relations. In the context of time-indexed systems~\citep{Peters_2013,Runge_2019,Runge_2020}, such graphs are commonly referred to as full-time DAGs (FT-DAGs), where each node corresponds to a time-indexed variable. In many applications, however, an FT-DAG cannot be specified directly from background knowledge~\citep{Ait_Bachir_2023}. As a result, researchers often rely either on abstractions of the underlying FT-DAG, such as summary causal graphs (SCGs)~\citep{Assaad_2023,Assaad_2024,Ferreira_2024} where each node corresponds to an entire time series,  or on causal discovery algorithms that recover partially oriented structures such as full time completed partially directed acyclic graphs (FT-CPDAGs)~\citep{Spirtes_2000,Chickering_2002}. Each of these two representations (SCG and FT-CPDAG) provides partial but complementary information about the underlying causal structure.

Since both SCGs and FT-CPDAGs contain less information than the true FT-DAG, it is natural to seek a representation that integrates their strengths. A natural candidate is the maximally oriented partially directed acyclic graph (FT-MPDAG)~\citep{Perkovic_2020}, which retains all orientations guaranteed either by discovery algorithms or by the background knowledge encoded in the SCG, thereby bringing the representation as close as possible to the FT-DAG.
However, relying on causal discovery algorithms comes at a cost: they are computationally demanding and require strong assumptions, such as faithfulness and causal sufficiency. This motivates a fundamental question: can we determine in advance\textemdash before running the causal discovery algorithm\textemdash whether a specific edge is guaranteed to be oriented, given the SCG and the available background knowledge? Answering this question allows researchers to focus discovery efforts where they are most informative.

This paper investigates the conditions under which causal relations in time series, particularly instantaneous ones, can be oriented  by combining information from SCGs with independence constraints extracted from data. Furthermore it discusses the implications of these results on causal reasoning and causal discovery.

The remainder of the paper is organized as follows. Section~\ref{sec:background} introduces the necessary background and related concepts. Section~\ref{sec:main} presents our main theoretical contributions. Section~\ref{sec:implications} presents the direct implications these results on causal reasoning. Finally, Section~\ref{sec:discussion} concludes with a discussion and outlines directions for future work.

\section{Background}
\label{sec:background}
We consider multivariate time series $\mathbb{V}$ evolving according to an \emph{unknown} discrete-time dynamic structural causal model (DT-DSCM)~\citep{Peters_2013,Runge_2020,Gerhardus_2020,Ferreira_2025}, observed from an initial time $t_0$ up to a final time $t_{\max}$, \ie, $\forall Y_t \in \mathbb{V}$,
$
    Y_t := f^{y_t}(\mathbb{X}, L^{y_t}_t),
$
where $\mathbb{X}$ represents the observed direct causes of $Y_t$ in $\mathbb{V}$ and $L^{y_t}_t$ a latent noise term for $Y_t$.
The DT-DSCM governs the temporal and causal dependencies among the variables in $\mathbb{V}$. 
It respects the temporal priority between causes and effects, ensuring that a variable $X_t \in \mathbb{V}$ cannot influence a variable $Y_{t-\ell}$ with $\ell >0$, while allowing for instantaneous relations, so that $X_t$ may directly affect $Y_t$.
The DT-DSCM induces a full time\footnote{The FT-DAG is a DAG; we add the prefix "FT" to emphasize that it represents the full-time structure of the dynamic causal model, capturing all temporal and instantaneous dependencies across the observed time series.} directed acyclic graph (FT-DAG)  $\ADMG=(\vADMG, \eADMG)$ with vertices $\vADMG$ and edges $\eADMG$.
In $\ADMG$, a parent of $Y_t \in \vADMG$ is any $X_{t'} \in \vADMG$ \st $X_{t'} \to Y_t$ is in $\eADMG$. The set of neighbors of $Y_t$ is denoted $Ne(Y_t,\ADMG)$ and the set of parents of $Y_t$ is denoted $Pa(Y_t,\ADMG)$.
A node $Y_t$ in $\ADMG$ is considered a collider on a path $p$ if there exists a subpath $Z_{t''} \rightarrow X_{t'} \leftarrow Y_t$ within $p$. In this context, we will interchangeably refer to the triple $Z_{t''} \rightarrow X_{t'} \leftarrow Y_t$ and the node $X_{t'}$ as the collider. Furthermore, $Z_{t''} \rightarrow X_{t'} \leftarrow Y_t$ is termed an unshielded collider if $Z_{t''}$ and $Y_t$ are not adjacent.

The DT-DSCM also induces a distribution compatible with $\ADMG$ according to the causal Markov condition~\citep{Spirtes_2000}.  

\begin{definition}[Causal Markov condition, \cite{Spirtes_2000}]
    Consider an FT-DAG $\ADMG$ compatible with a distribution $P$. For any variable $Y_t\in \vADMG$, $Y_t$ is independent in $P$ of its non-descendant conditional on its parents in $\ADMG$.
\end{definition}

In many applications, the FT-DAG cannot be fully specified from background knowledge~\citep{Assaad_2023,Ait_Bachir_2023}. Therefore, researchers in different fields focus on specifying based on background knowledge a summary causal graph, which is a summary of the FT-DAG where each vertex represent a time series.

\begin{definition}[Summary Causal Graph (SCG), \cite{Assaad_2024}]
	\label{def:SCG}
	Consider an FT-DAG $\ADMG = (\vADMG, \eADMG)$. The \emph{summary causal graph (SCG)} $\SCG = (\vSCG, \eSCG)$ compatible with $\ADMG$ is defined in the following way:
    \begin{itemize}
        \item 					$\mathbb{S} := \{S_\mathbb{Y} | \forall S_\mathbb{Y}=\{Y_{t_0}, \cdots, Y_{t_{max}}\} \in \mathbb{V}\}$,
        \item 					$\mathbb{E}^{\mathbbl{s}} := \{S_\mathbb{X}\rightarrow S_\mathbb{Y} | \forall S_\mathbb{X},S_\mathbb{Y} \in \mathbb{S},~\exists t'\leq t\in [t_0,t_{max}] \text{ \st } X_{t'}\rightarrow Y_{t}\in\mathbb{E}\}.$
\end{itemize}
\end{definition}

All graphical notions introduced for FT-DAGs, such as kinship and paths, are also applicable to SCGs. Then, as in an FT-DAG, a node $S_\mathbb{Z}$ is a collider between $S_\mathbb{X}$ and $S_\mathbb{Y}$ if it is their common child.  
Since the SCG can contains bidirected edges, we emphasize that configurations such as $S_\mathbb{X}\rightleftarrows S_\mathbb{Z}\rightleftarrows S_\mathbb{Y}$ or  
$S_\mathbb{X}\to S_\mathbb{Z}\rightleftarrows S_\mathbb{Y}$ are colliders. 
However, unlike FT-DAGs, SCGs can contain cycles of various sizes. 
We denote the special case of cycles which involve only one vertex in the SCG as \emph{self-loops}. 
\begin{definition}[Self-loop]
    Let $\SCG = (\vSCG, \eSCG)$ be an SCG and $S_\mathbb X\in \vSCG$ a variable. We say that $\SCG$ contains a \emph{self-loop} on $S_\mathbb X$ if $(S_\mathbb X,S_\mathbb X)\in \eSCG$.
\end{definition}

Many different FT-DAGs can be compatible with the same SCG. Examples of an SCG together with two of its compatible FT-DAGs are shown in Figure~\ref{fig:scg_ftdag}. 
We denote the set of all FT-DAGs compatible with a given SCG $\SCG$ as $\compatible{\SCG}$.

\begin{figure}[t!]
    \centering
    \begin{subfigure}[b]{0.12\textwidth}
        \centering
        \begin{tikzpicture}[{black, circle, draw, inner sep=0}]
            \tikzset{nodes={draw,rounded corners},minimum height=0.7cm,minimum width=0.7cm}
            \node (X) at (0,0) {$S_\mathbb X$};
            \node (Y) at (1.2,0) {$S_\mathbb Y$};
            \node (Z) at (0.6,-1) {$S_\mathbb Z$};
            \draw[->,>=latex] (Y) edge[bend left=15] (X);
            \draw[->,>=latex] (X) edge[bend left=15] (Y);
            \draw[->,>=latex] (Z) edge (Y);
            \draw[->,>=latex] (Z) edge[bend left=15] (X);
            \draw[->,>=latex] (X) edge[bend left=15] (Z);
            \draw[->,>=latex] (X) edge[loop above, looseness=1, min distance=5mm] (X);
            \draw[->,>=latex] (Z) edge[loop below, looseness=1, min distance=5mm] (Z);
        \end{tikzpicture}
        \caption{SCG $\SCG$.}
        \label{fig:lemma3scg}
    \end{subfigure}
    \hfill
    \begin{subfigure}[b]{0.16\textwidth}
        \centering
        \begin{tikzpicture}[{black, circle, draw, inner sep=0}]
            \tikzset{nodes={draw,rounded corners},minimum height=0.7cm,minimum width=0.7cm}

            \node[draw=none] at (-0.6,0) {...};
            \node[draw=none] at (1.6,0) {...};
        
            \node[draw=none] at (-0.6,-1) {...};
            \node[draw=none] at (1.6,-1) {...};
        
            \node[draw=none] at (-0.6,-2) {...};
            \node[draw=none] at (1.6,-2) {...};
            
            \node (X-1) at (0,0) {$X_{t-1}$};
            \node (X) at (1,0) {$X_{t}$};
            \node (Y-1) at (0,-1) {$Y_{t-1}$};
            \node (Y) at (1,-1) {$Y_{t}$};
            \node (Z-1) at (0,-2) {$Z_{t-1}$};
            \node (Z) at (1,-2) {$Z_{t}$};

            \draw[->,>=latex] (X-1) -- (X);
            \draw[->,>=latex] (X-1) -- (Z);
            \draw[->,>=latex] (Z-1) -- (Z);
            \draw[->,>=latex] (X-1) -- (Y);
            \draw[->,>=latex] (Y) -- (X);
            \draw[->,>=latex] (Y-1) -- (X-1);
            \draw[->,>=latex, color = blue] (X) to[bend left=40] (Z);
            \draw[->,>=latex, color = blue] (X-1) to[bend right=40] (Z-1);
            \draw[->,>=latex] (Z-1) -- (X);
            \draw[->,>=latex] (Z-1) -- (Y);
            
        \end{tikzpicture}
        \caption{FT-DAG $\ADMG_1$.}
        \label{fig:lemma3ftmpdag1}
    \end{subfigure}
    \hfill
    \begin{subfigure}[b]{0.16\textwidth}
        \centering
                \begin{tikzpicture}[{black, circle, draw, inner sep=0}]
            \tikzset{nodes={draw,rounded corners},minimum height=0.7cm,minimum width=0.7cm}
            
            \node[draw=none] at (-0.6,0) {...};
            \node[draw=none] at (1.6,0) {...};

            \node[draw=none] at (-0.6,-1) {...};
            \node[draw=none] at (1.6,-1) {...};
        
            \node[draw=none] at (-0.6,-2) {...};
            \node[draw=none] at (1.6,-2) {...};
    
            \node (X-1) at (0,0) {$X_{t-1}$};
            \node (X) at (1,0) {$X_{t}$};
            \node (Y-1) at (0,-1) {$Y_{t-1}$};
            \node (Y) at (1,-1) {$Y_{t}$};
            \node (Z-1) at (0,-2) {$Z_{t-1}$};
            \node (Z) at (1,-2) {$Z_{t}$};

            \draw[->,>=latex] (X-1) -- (X);
            \draw[->,>=latex] (Z-1) -- (Z);
            \draw[->,>=latex] (X-1) -- (Z);
            \draw[->,>=latex] (X-1) -- (Y);
            \draw[->,>=latex] (Y) -- (X);
            \draw[->,>=latex] (Y-1) -- (X-1);
            \draw[<-,>=latex, color = blue] (X) to[bend left=40] (Z);
            \draw[<-,>=latex, color = blue] (X-1) to[bend right=40] (Z-1);
            \draw[->,>=latex] (Z-1) -- (X);
            \draw[->,>=latex] (Z-1) -- (Y);
            
        \end{tikzpicture}
        \caption{FT-DAG $\ADMG_2$.}
        \label{fig:lemma3ftmpdag2}
    \end{subfigure}
    \caption{An SCG $\SCG$ and two FT-DAGs $\ADMG_1,\ADMG_2\in \compatible{\SCG}$: The blue edge is oriented differently in the two FT-DAGs.}
    \label{fig:scg_ftdag}
\end{figure}

SCGs can be very useful in many applications but they contain less information than the FT-DAG therefore it is not surprising that the set of causal effects identifiable from an SCG or from an FT-CPDAG is generally smaller than the set identifiable from FT-DAG~\citep{Assaad_2024,Ferreira_2024}.
This difficulty in SCG arises because of cycles.

Therefore, in contexts where it needed to acquire more information than the ones  provided by the SCG, \eg we would like to know if $X_t\rightarrow Y_t$, it can be useful to apply causal discovery algorithms to obtain a representation closer to the underlying DAG.
However, using causal discovery algorithms comes with a cost.  Causal discovery algorithms can be computationally expensive, and they require additional assumptions. In this paper, we focus on algorithms that require the following three assumptions. 
\begin{assumption}[Causal sufficiency, \cite{Spirtes_2000}]
\label{assumption:causal_sufficiency}
    All noise terms in the DT-DSCM are mutually independent and each noise term can affect only one observed variable.
\end{assumption}
\begin{assumption}[Faithfulness, \cite{Spirtes_2000}]
\label{assumption:faithfulness}
    All conditional independencies in the distribution are entails by the causal Markov condition.
\end{assumption}
\begin{assumption}[Stationarity]
\label{assumption:stationarity}
For any two temporal variables $Y_t$ and $Y_{t'}$ belonging to the same time series in a DT-DSCM, we have $f^{y_t} = f^{y_{t'}}$. In contrast, this equality does not necessarily hold for temporal variables $Y_t$ and $X_{t'}$ coming from two different time series in the DT-DSCM.
\end{assumption}

Assumption~\ref{assumption:stationarity} implies that, in an FT-DAG induced by a stationary DT-DSCM, if an edge $X_{t'} \to Y_t$ exists, then all edges $X_{t'-\ell} \to Y_{t-\ell}$ also exist for every $\ell \geq 0$ (\st $t'-\ell$ and $t-\ell$ remain within the interval $[t_0, t_{\max}]$). Consequently, FT-DAGs can be represented by focusing on a time window ranging from the variables at time $t$ to the variables at time $t-\gamma_{\max}$, where $\gamma_{\max}$ denotes the maximal lag for which a causal link exists between a variable at that lag and $t$. In the following examples, we set $\gamma_{\max}=1$ to simplify the FT-DAGs, although larger values are possible. To emphasize that we only depict a fragment of the FT-DAG rather than the complete graph, we always add ``\dots'' to indicate that, due to stationarity, the FT-DAG extends both into the past and the future beyond the displayed window of the FT-DAG. 
It is possible to relax the stationarity assumption by using independent ensembles of time series realizations (i.e., multiple multivariate time series). However, in this paper, we mainly focus on the case where only a single realization of the multivariate time series is available, which makes Assumption~\ref{assumption:stationarity} necessary. A weaker form of stationarity that accommodates multiple realizations is discussed in Section~\ref{sec:discussion}. 

Under Assumptions~\ref{assumption:causal_sufficiency}, \ref{assumption:faithfulness}, and \ref{assumption:stationarity}, the FT-DAG $\ADMG$ cannot be recovered; instead, one can guarantee identification of its FT-CPDAG~\citep{Chickering_2002}, which represents the Markov equivalence class of all FT-DAGs, denoted as $MEC(\ADMG)$, sharing the same set of conditional independencies~\citep{Verma_1990}. All DAGs in $MEC(\ADMG)$ are characterized by the same skeleton (i.e., the same unoriented graph over the variables) and the same unshielded colliders~\citep{Verma_1990}. When additional background knowledge is available, it becomes possible to orient more edges beyond those oriented in the FT-CPDAG. The resulting graph, which augments the FT-CPDAG with these orientations, is called a maximally oriented partially directed acyclic graph (MPDAG)~\citep{Perkovic_2020}. In this paper, the only source of background knowledge we assume is the information encoded in the summary causal graph (SCG). Accordingly, we define the FT-MPDAG\footnote{While the cited literature refers to CPDAGs and MPDAGs constructed from DAGs, we here use FT-CPDAGs and FT-MPDAGs to emphasize that they are built from FT-DAGs and thus encode temporal structure; this does not change their definitions or any results about them.} in terms of SCGs in the following.


\begin{definition}[FT-MPDAG using SCG]
\label{def:MPDAG}
Consider an SCG $\SCG$ and Markov equivalence class $MEC(\ADMG)$.
A full-time maximally oriented partially acyclic directed graph (FT-MPDAG) $\MPDAG=(\vMPDAG, \eMPDAG)$ defined using $\SCG$ and $MEC(\ADMG)$ is a graph with the same vertices as the original FT-DAG, with edges  defined as follows:
\begin{itemize}
    \item   $\mathbb{E}^{\rightarrow} 
    := \{\, X_{t'} \to Y_t \;\mid\; 
        \forall X_{t'},Y_t \in \mathbb{V}, \text{ if }$
        \begin{itemize}
            \item $\forall \mathcal{G}^i \in MEC(\ADMG),  X_{t'}\rightarrow Y_t \in \ADMG^i \text{ or }$
            \item   both the following holds:
            \begin{itemize}
                \item $\exists \mathcal{G}^i, \mathcal{G}^j \in MEC(\ADMG), \text{ \st } X\rightarrow Y \in \ADMG^i  \text{ and } Y\rightarrow X \in \ADMG^j$, \text{ and}
                \item $S_{\mathbb{X}} \rightarrow S_{\mathbb{Y}} \in \SCG \text{ and } S_{\mathbb{Y}} \rightarrow S_{\mathbb{X}} \not\in \SCG\}$, 
            \end{itemize}
        \end{itemize}
    \item $\mathbb{E}^{-} 
    := \{\, X_{t'} - Y_t \;\mid\;  
        \forall X_{t'},Y_t \in \mathbb{V}, \text{ if }$
        \begin{itemize}
            \item $\exists \mathcal{G}^i, \mathcal{G}^j \in MEC(\ADMG), \text{ \st } X\rightarrow Y \in \ADMG^i  \text{ and } Y\rightarrow X \in \ADMG^j \text{ and}$
            \item one of the following holds:
            \begin{itemize}
                \item $S_{\mathbb{X}} \rightarrow S_{\mathbb{Y}} \not\in SCG  \text{ and } S_{\mathbb{Y}} \rightarrow S_{\mathbb{X}} \not\in \SCG$, or
                \item $S_{\mathbb{X}} \rightarrow S_{\mathbb{Y}} \in SCG  \text{ and } S_{\mathbb{Y}} \rightarrow S_{\mathbb{X}} \in \SCG\}$,
            \end{itemize}
        \end{itemize}
\end{itemize}
where $\eMPDAG=\mathbb{E}^{\rightarrow}\cup \mathbb{E}^{-}$.
\end{definition}
As in the case of SCGs, all graphical notions introduced for FT-DAGs, applies to FT-MPDAG.
Under Assumptions~\ref{assumption:causal_sufficiency},~\ref{assumption:faithfulness} and~\ref{assumption:stationarity}, causal discovery algorithms can recover the FT-MPDAG from observational data.
One of the most widely used algorithms is the PC algorithm~\citep{Spirtes_2000,Colombo_2014}, which can be naturally extended to time series. We refer to this temporal extension as tPC. Several variants of tPC have been proposed in the literature~\citep{Spirtes_2000,Runge_2019,Runge_2020,Assaad_2022,Bang_2023}, each offering different advantages. For clarity and focus, we restrict our attention to the simplest version, which forms the foundation of these extensions. The only modification we introduce is that we present tPC explicitly in the presence of background knowledge provided by the SCG.

The tPC algorithm proceeds as follows:
It starts with a fully connected unoriented graph. It orients all lagged relations from past to present, and orients all instantaneous edges that can be directly resolved using the SCG (\ie, if $S_{\mathbb{X}} \rightarrow S_{\mathbb{Y}}$ but $S_{\mathbb{Y}} \not\rightarrow S_{\mathbb{X}}$, then $X_t \rightarrow Y_t$ is oriented).
Using conditional independence tests, it prunes edges from the graph: an edge between $X_{t'}$ and $Y_t$ is removed if there exists a conditioning set $\mathbb{Z}$ \st $Y_t$ and $X_{t'}$ are statistically independent given $\mathbb{Z}$. 
The above step guarantees detecting the correct skeleton of the graph. At this step, because of the temporal lag orientations (i.e. $X_{t'}\to Y_t$ is oriented whenever $t'<t$), the only edges which are unoriented are instantaneous edges (between $X_t$ and $Y_t$ at the same instant $t$).
Then, for each unshielded triple $Z_{t'} - X_t - Y_t$, it identifies an unshielded collider $Z_{t'} \rightarrow X_t \leftarrow Y_t$ if the middle node $X_t$ does not belong to the separating set that rendered $Y_t$ and $Z_t$ statistically independent. For clarity, we will refer to this rule as the \textbf{UC-Rule}. 
Then, it iteratively applies Meek’s four orientation rules~\citep{Meek_1995} to orient as much as possible the remaining unoriented edges. For clarity, in the following we only recall Rules~1 and 2 of Meek, which will be useful for the remainder of the paper.

    \text{\textbf{Meek-Rule 1:}} If $Z_{t'}\rightarrow X_t - Y_t$ is unshielded, then $Z_{t'}\rightarrow X_t \rightarrow Y_t$.

    \text{\textbf{Meek-Rule 2:}} If there is a directed path from $X_t$ to $Y_t$ and $X_t-Y_t$, then $X_t\to Y_t$. 

Finally, a subtle background knowledge rule implied by the SCG can be applied to all discovered FT-MPDAGs. By definition, an edge $S_\mathbb{X}\to S_\mathbb{Y}$ is present in the SCG if and only if there exist $t'$ and $t$ \st $X_{t'}\to Y_t$ in the FT-DAG. Hence, if there is a remaning unoriented instantaneous edge $X_t-Y_t$ at this step and no other oriented edge from $X_{t'}$ to $Y_t$, then by this definition of the SCG (and stationarity) we must have $X_t\to Y_t$; otherwise, the edge $S_\mathbb{X}\to S_\mathbb{Y}$ in the SCG would be contradicted.

The FT-DAG and the FT-MPDAG are guaranteed to share the same skeleton and the same unshielded colliders~\citep{Verma_1990}, along with additional edge orientations implied either by constrains from conditional independencies or by background knowledge.
Moreover, as we mentionned, due to the temporal priority of causes over their effects, all lagged relations are oriented identically in both the FT-DAG and the FT-MPDAG.
By definition, the only edges that may remain unoriented in the FT-MPDAG are the instantaneous ones.
So if the goal is to determine whether an edge such as $X_{t-1}\rightarrow Y_t$ exists in the underlying DAG, then applying these algorithms is certainly worthwhile. However, when the question concerns edges of the form $X_t \rightarrow Y_t$, the situation becomes less straightforward as after running causal discovery, it is possible that the FT-MPDAG only indicates an unoriented edge $X_t - Y_t$, leaving the causal direction unresolved.

In this paper, we aim to characterize the orientations 
of these instantaneous edges by exploiting the background knowledge provided by an SCG.
To this end, we introduce a simple operator that captures the status of the causal orientation between two nodes in an FT-DAG or in an FT-MPDAG:
\begin{align*}
\text{status}(X_t, Y_t; \MPDAG) =
\begin{cases}
X_t \to Y_t & \text{if } X_t \to Y_t \in \MPDAG,\\
Y_t \to X_t & \text{if } Y_t \to X_t \in \MPDAG,\\
1 & \text{if } X_t - Y_t \text{ in }\MPDAG,\\
0 & \text{otherwise}.
\end{cases}
\end{align*}

The central concept underlying our characterization is the orientability of edges from an SCG.

\begin{definition}[Orientability of an edge from SCGs and faithful distributions compatible with an FT-DAG]
Let $\SCG$ be an SCG, let $\ADMG \in \compatible{\SCG}$, and let $\mathbb{P}$ denote the set of faithful distributions compatible with $\ADMG$.  
The direct relation between $X_t$ and $Y_t$, i.e., $\text{status}(X_t, Y_t; \ADMG)$, is said to be \emph{orientable} from $(\SCG, \mathbb{P})$ if
$$\forall P \in \mathbb{P}, \quad \text{status}(X_t, Y_t; \MPDAG)\ne 1,$$
where $\MPDAG$ denotes the FT-MPDAG constructed from $P$ and $\SCG$.
\end{definition}

Since there might exists many FT-DAGs compatible with an SCG the set of distributions compatible with an SCG is larger than the set of distributions compatible with a DAG. Consequently, the standard definition of orientability is not sufficient for our purposes. In this paper, we assume that the underlying FT-DAG is unknown, and thus we cannot classify a distribution as being compatible with any specific FT-DAG. Instead, we must account for all distributions compatible with the SCG. This distinction motivates the notion of \emph{$s$-orientability}.

\begin{definition}[$s$-orientability of an edge from an SCG and all compatible faithful distributions]
Let $\SCG$ be an SCG, let $\ADMG \in \compatible{\SCG}$, and let $\mathbb{P}^*$ denote the set of faithful distributions compatible with at least one $\ADMG\in \compatible{\SCG}$.  
The direct relation between $X_t$ and $Y_t$, i.e., $\text{status}(X_t, Y_t; \ADMG)$, is said to be \emph{$s$-orientable} from $(\SCG, \mathbb{P}^*)$ if for every $\ADMG\in \compatible{\SCG}$ and every set of distribution $\mathbb{P}\subset\mathbb{P}^*$ compatible with  $\ADMG$, $\text{status}(X_t, Y_t; \ADMG)$ is orientable from $(\SCG, \mathbb{P})$, \ie,
$$\forall P \in \mathbb{P}^*, \quad \text{status}(X_t, Y_t; \MPDAG)\ne 1,$$
where $\MPDAG$ denotes the FT-MPDAG constructed from $P$ and $\SCG$. 
\end{definition}

We emphasize that we define $s$-orientability as \emph{$s$-orientability from the pair $(\SCG, \mathbb P^*)$} to highlight that edge orientations can only be determined given a faithful distribution compatible with $\SCG$. However, since the set of faithful distributions compatible with any FT-DAG compatible with $\SCG$ depends solely on $\SCG$, $s$-orientability is fundamentally a property of $\SCG$ itself, independent of any specific distribution. While we retain the original notation to stress the role of faithfulness throughout the paper, one could equivalently refer to it as \emph{$s$-orientability from $\SCG$} alone.




\section{Main results}
\label{sec:main}

Starting from Definition~\ref{def:SCG}, it is clear that an oriented (non-bidirected) edge in the SCG is highly informative. Indeed, it indicates that any corresponding edge in the FT-MPDAG is oriented in the same direction. Thus, if an instantaneous edge exists between $X_t$ and $Y_t$, it is oriented consistently with the corresponding macro-variables in the SCG. Lemma~\ref{lemma:directed_edge} formalizes this intuitive result.

\begin{lemma}\label{lemma:directed_edge}
Let $\SCG$ be an SCG  and $\mathbb{P}^*$ the set of distributions compatible with $\SCG$.  
Under Assumptions~\ref{assumption:causal_sufficiency},~\ref{assumption:faithfulness} and~\ref{assumption:stationarity}, if $S_\mathbb{X} \to S_\mathbb{Y} \in \SCG$, then $\forall \ADMG\in\compatible{\SCG}$,  $\text{status}(X_t,Y_t;\ADMG)$ is $s$-orientable from $(\SCG, \mathbb{P}^*)$.
\end{lemma}

\begin{proof}
    Immediate consequence of Definition~\ref{def:MPDAG}.
\end{proof}

We thus obtain that the only potentially unoriented edges in the FT-MPDAG are instantaneous edges of the form $X_t - Y_t$, whose associated macro-variables are bidirected in the SCG, i.e., $S_\mathbb{X} \rightleftarrows S_\mathbb{Y}$. The following lemma further reduces these cases by introducing a condition on the self-loops of the macro-variables.

\begin{lemma}\label{lemma:bidir_singleSL}
Let $\SCG$ be an SCG  and $\mathbb{P}^*$ the set of distributions compatible with $\SCG$.  
Under Assumptions~\ref{assumption:causal_sufficiency},~\ref{assumption:faithfulness} and~\ref{assumption:stationarity}, if $S_\mathbb{X} \rightleftarrows S_\mathbb{Y} \in \SCG$ and there is no self-loop on \textbf{both} $S_\mathbb{X}$ \textbf{and} $S_\mathbb{Y}$ \textbf{simultaneously}, then $\forall \ADMG\in\compatible{\SCG}$,  $\text{status}(X_t,Y_t;\ADMG)$ is $s$-orientable from $(\SCG, \mathbb{P}^*)$.
\end{lemma}

\begin{proof}
If there is no edge between $X_t$ and $Y_t$ in $\ADMG$, then no edge exists between them in $\MPDAG$,  
since $\MPDAG$ shares the same skeleton as $\ADMG$.

Suppose there is an edge between $X_t$ and $Y_t$ in $\MPDAG$. We show it must be oriented. First, assume there are no self-loops. Since $S_\mathbb{X} \rightleftarrows S_\mathbb{Y} \in \SCG$, there exist $t_1,t_2$ \st $t-\gamma_{\max}\le t_1, t_2 \le t$, and $Y_{t_1} \to X_t\in \MPDAG$ and $X_{t_2} \to Y_t\in \MPDAG$; and since $\MPDAG$ is acyclic, $t_1<t$ or $t_2<t$. Assume without loss of generality that $t_1<t$. Then, because of temporal ordering $Y_{t_1} \to X_t \in \MPDAG$ and $(Y_{t_1}, X_t, Y_t)$ is an unshielded triple, since $S_\mathbb{Y}$ has no self-loop and thus $Y_{t_1}$ and $Y_t$ are not adjacent. This triple is either an unshielded collider $Y_{t_1} \to X_t \leftarrow Y_t$, which is oriented in $\MPDAG$ by UC-rule, or a chain $Y_{t_1} \to X_t \to Y_t$, which is oriented in $\MPDAG$ by Meek-Rule 1. Now, suppose there is a self-loop on one of the nodes. Without loss of generality, let $S_\mathbb{Y}$ have a self-loop. Then, there exists $t_1 < t$ \st $Y_{t_1} \to Y_t$ in the FT-MPDAG. Two cases arise: (i) there is no $t_2$ \st $t-\gamma_{\max}\le t_2<t$ and $X_{t_2} \to Y_t$ so necessarily $X_t \to Y_t$ to satisfy the SCG constraint $S_\mathbb{X} \rightleftarrows S_\mathbb{Y}$, which requires at least one edge between a time point of $\{X_{t_0},\cdots, X_{t_{\max}}\}$ and one of $\{Y_{t_0},\cdots, Y_{t_{\max}}\}$; or (ii) there exists $t_3$ \st $t-\gamma_{\max}\le t_3< t$ and $X_{t_3} \to Y_t$ so $(X_{t_3}, X_t, Y_t)$ forms an unshielded triplet (since $S_\mathbb{X}$ has no self-loop), and the edge $X_t$–$Y_t$ is oriented either by Meek-Rule 1 if $X_t \leftarrow Y_t\leftarrow X_{t_3}$, or by the UC-rule if $X_t \to Y_t\leftarrow X_{t_3}$.
\end{proof}

\begin{figure}[t!]
    \centering
    \begin{subfigure}[b]{0.11\textwidth}
        \centering
        \begin{tikzpicture}[{black, circle, draw, inner sep=0}]
            \tikzset{nodes={draw,rounded corners},minimum height=0.7cm,minimum width=0.7cm}
            \node (X) at (0,0) {$S_\mathbb X$};
            \node (Y) at (1.2,0) {$S_\mathbb Y$};
            \draw[->,>=latex] (Y) edge[bend left=15] (X);
            \draw[->,>=latex] (X) edge[bend left=15] (Y);
        \end{tikzpicture}
        \caption{SCG $\SCG_1$.}
    \end{subfigure}
    \hfill
    \begin{subfigure}[b]{0.17\textwidth}
        \centering
        \begin{tikzpicture}[{black, circle, draw, inner sep=0}]
            \tikzset{nodes={draw,rounded corners},minimum height=0.7cm,minimum width=0.7cm}

            \node[draw=none] at (-0.6,0) {...};
            \node[draw=none] at (1.6,0) {...};
        
            \node[draw=none] at (-0.6,-1) {...};
            \node[draw=none] at (1.6,-1) {...};

            \node (X-1) at (0,0) {$X_{t-1}$};
            \node (X) at (1,0) {$X_{t}$};
            \node (Y-1) at (0,-1) {$Y_{t-1}$};
            \node (Y) at (1,-1) {$Y_{t}$};
            \draw[->,>=latex, color = red] (Y) -- (X);
            \draw[->,>=latex, color = blue] (Y-1) -- (X-1);
            \draw[->,>=latex, color = red] (X-1) -- (Y);
        \end{tikzpicture}
        \caption{FT-MPDAG $\MPDAG_1^1$.}
    \end{subfigure}
    \hfill
    \begin{subfigure}[b]{0.17\textwidth}
        \centering
        \begin{tikzpicture}[{black, circle, draw, inner sep=0}]
            \tikzset{nodes={draw,rounded corners},minimum height=0.7cm,minimum width=0.7cm}

            \node[draw=none] at (-0.6,0) {...};
            \node[draw=none] at (1.6,0) {...};
        
            \node[draw=none] at (1.6,-1) {...};
            \node[draw=none] at (-0.6,-1) {...};

            \node (X-1) at (0,0) {$X_{t-1}$};
            \node (X) at (1,0) {$X_{t}$};
            \node (Y-1) at (0,-1) {$Y_{t-1}$};
            \node (Y) at (1,-1) {$Y_{t}$};
            \draw[->,>=latex, color = red] (Y) -- (X);
            \draw[->,>=latex, color = red] (Y-1) -- (X);
            \draw[->,>=latex, color = blue] (Y-1) -- (X-1);
        \end{tikzpicture}
        \caption{FT-MPDAG $\MPDAG_2^1$.}
    \end{subfigure}

    \vspace{0.3cm}
    \begin{subfigure}[b]{0.11\textwidth}
        \centering
        \begin{tikzpicture}[{black, circle, draw, inner sep=0}]
            \tikzset{nodes={draw,rounded corners},minimum height=0.7cm,minimum width=0.7cm}
            \node (X) at (0,0) {$S_\mathbb X$};
            \node (Y) at (1.2,0) {$S_\mathbb Y$};
            \draw[->,>=latex] (Y) edge[bend left=15] (X);
            \draw[->,>=latex] (Y) edge[loop above, looseness=1, min distance=5mm] (Y);
            \draw[->,>=latex] (X) edge[bend left=15] (Y);
        \end{tikzpicture}
        \caption{SCG $\SCG_2$.}
    \end{subfigure}
    \hfill
    \begin{subfigure}[b]{0.17\textwidth}
        \centering
        \begin{tikzpicture}[{black, circle, draw, inner sep=0}]
            \tikzset{nodes={draw,rounded corners},minimum height=0.7cm,minimum width=0.7cm}

            \node[draw=none] at (-0.6,0) {...};
            \node[draw=none] at (1.6,0) {...};
        
            \node[draw=none] at (-0.6,-1) {...};
            \node[draw=none] at (1.6,-1) {...};

            \node (X-1) at (0,0) {$X_{t-1}$};
            \node (X) at (1,0) {$X_{t}$};
            \node (Y-1) at (0,-1) {$Y_{t-1}$};
            \node (Y) at (1,-1) {$Y_{t}$};
            \draw[->,>=latex, color = red] (X) -- (Y);
            \draw[->,>=latex, color = blue] (X-1) -- (Y-1);
            \draw[->,>=latex] (Y-1) -- (Y);
            \draw[->,>=latex] (Y-1) -- (X);
        \end{tikzpicture}
        \caption{FT-MPDAG $\MPDAG_1^2$.}
    \end{subfigure}
    \hfill
    \begin{subfigure}[b]{0.17\textwidth}
        \centering
        \begin{tikzpicture}[{black, circle, draw, inner sep=0}]
            \tikzset{nodes={draw,rounded corners},minimum height=0.7cm,minimum width=0.7cm}

            \node[draw=none] at (-0.6,0) {...};
            \node[draw=none] at (1.6,0) {...};
        
            \node[draw=none] at (-0.6,-1) {...};
            \node[draw=none] at (1.6,-1) {...};

            \node (X-1) at (0,0) {$X_{t-1}$};
            \node (X) at (1,0) {$X_{t}$};
            \node (Y-1) at (0,-1) {$Y_{t-1}$};
            \node (Y) at (1,-1) {$Y_{t}$};
            \draw[->,>=latex, color = red] (Y) -- (X);
            \draw[->,>=latex, color = blue] (Y-1) -- (X-1);
            \draw[->,>=latex, color = red] (X-1) -- (Y);  
            \draw[->,>=latex] (Y-1) -- (Y);
            \draw[->,>=latex] (Y-1) -- (X);
        \end{tikzpicture}
        \caption{FT-MPDAG $\MPDAG_2^2$.}
    \end{subfigure}

    \caption{Visual illustration of Lemma~\ref{lemma:bidir_singleSL}, with the top FT-MPDAGs based on the top SCG $\SCG_1$ (no self-loop) and the bottom FT-MPDAGs based on the bottom SCG $\SCG_2$ (with a self-loop). Red illustrates the edges used to orient the edge $X_t-Y_t$, blue is the repetition of this edge by stationarity.}
    \label{fig:lemma2_visual_combined}
\end{figure}

Figure~\ref{fig:lemma2_visual_combined} provides an intuitive visual illustration of Lemma~\ref{lemma:bidir_singleSL}. The figures are separated into two cases: one corresponds to an SCG $\SCG_1$ without any self-loop (top) and one corresponds to an SCG $\SCG_2$ with a single self-loop (bottom). In each case, the SCG is shown alongside two examples of compatible FT-MPDAGs, ($\MPDAG_1^1$ and $\MPDAG_2^1$ for the case without self-loop, $\MPDAG_1^2$ and $\MPDAG_2^2$, for the case with self-loop).
In $\MPDAG_1^1$, the edge between $X_t$ and $Y_t$ is necessarily oriented according to Meek-Rule 1, $Y_{t-1} \to Y_t \to X_t$ (given that $Y_{t-1}\to Y_t$ is oriented by temporal order). In $\MPDAG_2^1$, the edge between $X_t$ and $Y_t$ is oriented due to the unshielded collider $Y_{t-1}\to X_t \leftarrow Y_t$. Note that this Meek rule and the unshielded collider are applied because the triplet $(Y_{t-1}, X_t, Y_t)$ is always unshielded due to the absence of self-loops. In $\MPDAG_1^2$, if the edge between $X_t$ and $Y_t$ were oriented in the opposite direction, $Y_t \leftarrow X_t$, this would contradict the SCG $\SCG$, as no edge would then point from $S_\mathbb{X}$ to $S_\mathbb{Y}$. Therefore, the edge must be oriented as $X_t \to Y_t$ in all compatible DAGs and hence in the FT-MPDAG. In $\MPDAG_2^2$, the situation is more similar to the case without self-loops: the edge between $X_{t-1}$ and $Y_t$ is oriented according to Meek-Rule 1, $X_{t-1} \to Y_t \to X_t$.

From Lemma~\ref{lemma:directed_edge} and~\ref{lemma:bidir_singleSL}, it immediately follows that for an SCG $\SCG$ without self-loops, any FT-MPDAG $\MPDAG$ based on $\SCG$ contains no unoriented edges (it is an FT-DAG).

At this stage, the only potentially unoriented edges in the FT-MPDAG are instantaneous edges whose corresponding macro-variables that are bidirected in the SCG and each contain a self-loop. Lemma~\ref{lemma:bidir+UC} further restricts the possibility of unoriented edges in the FT-MPDAG by stating that such edges are oriented whenever there exists a node that is a parent of one macro-node but not of the other.

\begin{lemma}\label{lemma:bidir+UC}
Let $\SCG$ be an SCG and $\mathbb{P}^*$ the set of distributions compatible with $\SCG$.  
Under Assumptions~\ref{assumption:causal_sufficiency},~\ref{assumption:faithfulness} and~\ref{assumption:stationarity}, if $S_\mathbb{X} \rightleftarrows S_\mathbb{Y}\in \SCG$ and $\exists S_\mathbb Z\in Pa(S_\mathbb X,\SCG)\setminus \{S_\mathbb X\}$ \st $S_\mathbb Z\notin Pa(S_\mathbb Y,\SCG)$, then $\forall \ADMG\in\compatible{\SCG}$,  $\text{status}(X_t,Y_t;\ADMG)$ is $s$-orientable from $(\SCG, \mathbb{P}^*)$.
\end{lemma}

\begin{proof}
If there is no edge between $X_t$ and $Y_t$ in $\ADMG$, then no edge exists between them in $\MPDAG$, since $\MPDAG$ shares the same skeleton as $\ADMG$.

Suppose that there is an edge between $X_t$ and $Y_t$ in $\MPDAG$. We show that such an edge must be oriented. First, since $S_{\mathbb Z}\in Pa(S_{\mathbb X},\SCG)$, there exists $t' \le t$ such that $Z_{t'} \to X_t$. If $t' < t$, the triple $(Z_{t'},X_t,Y_t)$ is unshielded because $Y_t \to Z_{t'}$ is impossible due to the temporal order. Hence this triple either forms a chain, which is oriented by Meek's rule~1 in $\MPDAG$, or an unshielded collider, which is also oriented by the UC-Rule in $\MPDAG$. If $t'=t$ and there is no $t''<t$ \st $Z_{t''}\to X_t$, the edge $Z_{t}\to X_t$ is necessarily oriented in $\MPDAG$ (otherwise it would contradict $S_{\mathbb Z}\in Pa(S_{\mathbb X},\SCG)$). Then we distinguish two case. If $Y_t$ is not adjacent to $Z_t$ in $\MPDAG$, the unshielded triple $(Z_t,X_t,Y_t)$ can be oriented by the same argument as in the case $t'< t$. If instead $Y_t$ is adjacent to $Z_t$, then we necessarily have $Y_t \to Z_t$ since $S_\mathbb Z\notin Pa(S_\mathbb Y,\SCG)$. To avoid creating a cycle (Meek's rule~2), since $Y_t \to Z_t \to X_t$, the edge $Y_t \to X_t$ must be oriented in $\MPDAG$.
\end{proof}


\begin{figure}[t!]
    \centering
    \begin{subfigure}[b]{0.11\textwidth}
        \centering
        \begin{tikzpicture}[{black, circle, draw, inner sep=0}]
            \tikzset{nodes={draw,rounded corners},minimum height=0.7cm,minimum width=0.7cm}
            \node (X) at (0,0) {$S_\mathbb X$};
            \node (Y) at (1.5,0) {$S_\mathbb Y$};
            \node (Z) at (0.75,-1) {$S_\mathbb Z$};
            \draw[->,>=latex] (Y) edge[bend left=15] (X);
            \draw[->,>=latex] (X) edge[bend left=15] (Y);
            \draw[->,>=latex] (Z) edge (X);
            \draw[->,>=latex] (X) edge[loop above, looseness=1, min distance=5mm] (X);
            \draw[->,>=latex] (Y) edge[loop above, looseness=1, min distance=5mm] (Y);
            \draw[->,>=latex] (Z) edge[loop below, looseness=1, min distance=5mm] (Z);
        \end{tikzpicture}
        \caption{SCG $\SCG_1$.}
        \label{fig:lemma3scg}
    \end{subfigure}
    \hfill
    \begin{subfigure}[b]{0.17\textwidth}
        \centering
        \begin{tikzpicture}[{black, circle, draw, inner sep=0}]
            \tikzset{nodes={draw,rounded corners},minimum height=0.7cm,minimum width=0.7cm}

            \node[draw=none] at (-0.6,0) {...};
            \node[draw=none] at (1.6,0) {...};
        
            \node[draw=none] at (-0.6,-1) {...};
            \node[draw=none] at (1.6,-1) {...};
        
            \node[draw=none] at (-0.6,-2) {...};
            \node[draw=none] at (1.6,-2) {...};
            
            \node (X-1) at (0,0) {$X_{t-1}$};
            \node (X) at (1,0) {$X_{t}$};
            \node (Y-1) at (0,-1) {$Y_{t-1}$};
            \node (Y) at (1,-1) {$Y_{t}$};
            \node (Z-1) at (0,-2) {$Z_{t-1}$};
            \node (Z) at (1,-2) {$Z_{t}$};

            \draw[->,>=latex] (X-1) -- (X);
            \draw[->,>=latex] (Z-1) -- (Z);
            \draw[->,>=latex] (Z-1) to[bend left=40] (X-1);
            \draw[->,>=latex, red] (Z) to[bend right=40] (X);
            \draw[->,>=latex] (Y-1) -- (Y);
            \draw[->,>=latex] (Y-1) -- (X);
            \draw[->,>=latex, color = red] (Y) -- (X);
            \draw[->,>=latex, color = blue] (Y-1) -- (X-1);
            \draw[->,>=latex] (X-1) -- (Y);
        \end{tikzpicture}
        \caption{FT-MPDAG $\MPDAG_1^1$.}
        \label{fig:lemma3ftmpdag1}
    \end{subfigure}
    \hfill
    \begin{subfigure}[b]{0.17\textwidth}
        \centering
        \begin{tikzpicture}[{black, circle, draw, inner sep=0}]
            \tikzset{nodes={draw,rounded corners},minimum height=0.7cm,minimum width=0.7cm}

            \node[draw=none] at (-0.6,0) {...};
            \node[draw=none] at (1.6,0) {...};
        
            \node[draw=none] at (-0.6,-1) {...};
            \node[draw=none] at (1.6,-1) {...};
        
            \node[draw=none] at (-0.6,-2) {...};
            \node[draw=none] at (1.6,-2) {...};
            
            \node (X-1) at (0,0) {$X_{t-1}$};
            \node (X) at (1,0) {$X_{t}$};
            \node (Y-1) at (0,-1) {$Y_{t-1}$};
            \node (Y) at (1,-1) {$Y_{t}$};
            \node (Z-1) at (0,-2) {$Z_{t-1}$};
            \node (Z) at (1,-2) {$Z_{t}$};

            \draw[->,>=latex] (X-1) -- (X);
            \draw[->,>=latex] (Z-1) -- (Z);
            \draw[->,>=latex] (Y-1) -- (Y);
            \draw[->,>=latex] (Y-1) -- (X);
            \draw[->,>=latex, color = red] (Z-1) -- (X);
            \draw[->,>=latex, color = red] (X) -- (Y);
            \draw[->,>=latex, color = blue] (X-1) -- (Y-1);
            \draw[->,>=latex] (X-1) -- (Y);
        \end{tikzpicture}
        \caption{FT-MPDAG $\MPDAG_2^1$.}
        \label{fig:lemma3ftmpdag2}
    \end{subfigure}
    \label{fig:lemma3visual}
\hfill
        \begin{subfigure}[b]{0.11\textwidth}
        \centering
        \begin{tikzpicture}[{black, circle, draw, inner sep=0}]
            \tikzset{nodes={draw,rounded corners},minimum height=0.7cm,minimum width=0.7cm}
            \node (X) at (0,0) {$S_\mathbb X$};
            \node (Y) at (1.5,0) {$S_\mathbb Y$};
            \node (Z) at (0.75,-1) {$S_\mathbb Z$};
            \draw[->,>=latex] (Y) edge[bend left=15] (X);
            \draw[->,>=latex] (X) edge[bend left=15] (Y);
             \draw[->,>=latex] (Y) edge (Z);
              \draw[->,>=latex] (X) edge[bend left=15] (Z);
               \draw[->,>=latex] (Z) edge[bend left=15] (X);
             \draw[->,>=latex] (Y) edge[loop above, looseness=1, min distance=5mm] (Y);
            \draw[->,>=latex] (X) edge[loop above, looseness=1, min distance=5mm] (X);
            \draw[->,>=latex] (Z) edge[loop below, looseness=1, min distance=5mm] (Z);
        \end{tikzpicture}
        \caption{SCG $\SCG_2$.}
        \label{fig:lemma3scg}
    \end{subfigure}
    \hfill
    \begin{subfigure}[b]{0.17\textwidth}
        \centering
        \begin{tikzpicture}[{black, circle, draw, inner sep=0}]
            \tikzset{nodes={draw,rounded corners},minimum height=0.7cm,minimum width=0.7cm}

            \node[draw=none] at (-0.6,0) {...};
            \node[draw=none] at (1.6,0) {...};
        
            \node[draw=none] at (-0.6,-1) {...};
            \node[draw=none] at (1.6,-1) {...};
        
            \node[draw=none] at (-0.6,-2) {...};
            \node[draw=none] at (1.6,-2) {...};
            
            \node (X-1) at (0,0) {$X_{t-1}$};
            \node (X) at (1,0) {$X_{t}$};
            \node (Y-1) at (0,-1) {$Y_{t-1}$};
            \node (Y) at (1,-1) {$Y_{t}$};
            \node (Z-1) at (0,-2) {$Z_{t-1}$};
            \node (Z) at (1,-2) {$Z_{t}$};

            \draw[->,>=latex] (X-1) -- (X);
            \draw[->,>=latex] (Z-1) -- (Z);
            \draw[->,>=latex] (Y-1) -- (X);
            \draw[->,>=latex] (Y-1) -- (Y);
            \draw[->,>=latex] (Y-1) -- (Z);
            \draw[->,>=latex, red] (Z-1) -- (X);
            \draw[->,>=latex, color = red] (Y) -- (X);
            \draw[->,>=latex, color = blue] (Y-1) -- (X-1);
            \draw[->,>=latex] (X-1) -- (Y);
        \end{tikzpicture}
        \caption{FT-MPDAG $\MPDAG_1^2$.}
        \label{fig:lemma3ftmpdag3}
    \end{subfigure}
    \hfill
    \begin{subfigure}[b]{0.17\textwidth}
        \centering
        \begin{tikzpicture}[{black, circle, draw, inner sep=0}]
            \tikzset{nodes={draw,rounded corners},minimum height=0.7cm,minimum width=0.7cm}

            \node[draw=none] at (-0.6,0) {...};
            \node[draw=none] at (1.6,0) {...};
        
            \node[draw=none] at (-0.6,-1) {...};
            \node[draw=none] at (1.6,-1) {...};
        
            \node[draw=none] at (-0.6,-2) {...};
            \node[draw=none] at (1.6,-2) {...};
            
            \node (X-1) at (0,0) {$X_{t-1}$};
            \node (X) at (1,0) {$X_{t}$};
            \node (Y-1) at (0,-1) {$Y_{t-1}$};
            \node (Y) at (1,-1) {$Y_{t}$};
            \node (Z-1) at (0,-2) {$Z_{t-1}$};
            \node (Z) at (1,-2) {$Z_{t}$};

            \draw[->,>=latex] (X-1) -- (X);
            \draw[->,>=latex] (Z-1) -- (Z);
            \draw[->,>=latex] (Y-1) -- (Z-1);
            \draw[->,>=latex] (Y-1) -- (X);
            \draw[->,>=latex, color = red] (X) -- (Y);
            \draw[->,>=latex, color = blue] (X-1) -- (Y-1);
            \draw[->,>=latex] (X-1) -- (Y);
            \draw[->,>=latex, color = red] (Y) -- (Z);
            \draw[->,>=latex] (Z-1) to[bend left=40] (X-1);
            \draw[->,>=latex, color = red] (Z) to[bend right=40] (X);
        \end{tikzpicture}
        \caption{FT-MPDAG $\MPDAG_2^2$.}
        \label{fig:lemma3ftmpdag4}
    \end{subfigure}
     \caption{Visual illustration of Lemma~\ref{lemma:bidir+UC}, with the top FT-MPDAGs based on the top SCG $\SCG_1$ (unshielded collider) and the bottom FT-MPDAGs based on the bottom SCG $\SCG_2$ (shielded collider but with $S_\mathbb Z$ not a parent of $S_\mathbb Y$). Red illustrates the edges used to orient the edge $X_t-Y_t$, blue is the repetition of this edge by stationarity.}
     \label{fig:lemma3visual}
\end{figure}

For a visual illustration of Lemma~\ref{lemma:bidir+UC}, consider the example given in Figure~\ref{fig:lemma3visual} consisting of two SCG $\SCG_1$ and $\SCG_2$ where $S_\mathbb X\leftrightarrows S_\mathbb Y$ satisfies condition of Lemma~\ref{lemma:bidir+UC}, and two compatible FT-MPDAGs for each. The SCG $\SCG_1$ contains an unshielded collider $S_\mathbb Z\to S_\mathbb X\leftrightarrows S_\mathbb Y$ and the SCG $\SCG_2$ contains a shielded collider $S_\mathbb Z\leftrightarrows S_\mathbb X\leftrightarrows S_\mathbb Y$ but with $S_\mathbb Z\notin Pa(S_\mathbb Y,\SCG_2)$. First, consider $\SCG_1$ : (i) in $\MPDAG_1^1$, the edge between $X_t$ and $Y_t$ is oriented by UC-rule applied to the triplet $Z_t \to X_t \leftarrow Y_t$, where $Z_t \to X_t$ is oriented by background knowledge from $S_\mathbb Z \to S_\mathbb X$ in $\SCG_1$; and (ii) in $\MPDAG_2^1$, the orientation arises from Meek-rule 1 applied to the chain $Z_{t-1} \to X_t \to Y_t$ with $Z_{t-1}\to X_t$ oriented by time. Now, consider $\SCG_2$: (i) in $\MPDAG_1^2$, the edge between $X_t$ and $Y_t$ is oriented by UC-rule applied to the triplet $Z_{t-1} \to X_t \leftarrow Y_t$; and (ii) in $\MPDAG_2^1$, the orientation arises from Meek-rule 2 applied to the triple $Z_t \to X_t \to Y_t$, with $Y_t\to Z_t$ oriented by background knowledge from $S_\mathbb Y \to S_\mathbb Z$ in $\SCG_2$ and $Z_t\to X_t$ oriented to satisfy the constraint $S_\mathbb Z\in Pa(S_\mathbb X,\SCG_2)$ since there is no $t'<t$ such that $Z_{t'}\to X_t$ in $\MPDAG$.

\color{black}

We have now explored all SCG configurations that ensure 
$s$-orientability, as emphasized by Theorem~\ref{th:theorem1} which summarizes these results and establishes their completeness.

\begin{theorem}[$s$-orientability of orientations in the FT-MPDAG from the SCG and faithful distribution]\label{th:theorem1}
Let $\SCG$ be an SCG and $\mathbb{P}^*$ the set of distributions compatible with $\SCG$. 
Under Assumptions~\ref{assumption:causal_sufficiency}, \ref{assumption:faithfulness} and \ref{assumption:stationarity}, $\forall\ADMG\in \compatible{\SCG}$, $\text{status}(X_t,Y_t;\ADMG)$ is \textbf{not} $s$-orientable from $(\SCG, \mathbb{P}^*)$ \emph{if and only if}:
\begin{enumerate}
    \item $S_\mathbb{X} \rightleftarrows S_\mathbb{Y} \in \SCG$, and
    \item both $S_\mathbb{X}$ and $S_\mathbb{Y}$ have self-loops simultaneously in $\SCG$, and
    \item $Pa(S_\mathbb X, \SCG)=Pa(S_\mathbb Y, \SCG).$
\end{enumerate}
\end{theorem}

\begin{proof}

Lemmas~\ref{lemma:directed_edge},~\ref{lemma:bidir_singleSL} and~\ref{lemma:bidir+UC} establish immediately that if at least one of the items of the theorem is not satisfied then $\text{status}(X_t,Y_t;\ADMG)$ is  $s$-orientable. 

In the following, we prove that if the three items of the theorem are satisfied then $\text{status}(X_t,Y_t;\ADMG)$ is  not $s$-orientable.
Let $\SCG$ be an SCG, and let $S_{\mathbb X}$ and $S_{\mathbb Y}$ be two nodes of $\SCG$ satisfying the three items of the theorem. We show that this implies the existence of an FT-MPDAG $\MPDAG$ compatible with $\SCG$ and a faithful distribution $P\in\mathbb P^*$ such that, for every $\ADMG$ compatible with $\MPDAG$, we have $\mathrm{status}(X_t,Y_t;\ADMG)=1$, that is, the instantaneous edge $X_t - Y_t$ remains unoriented in $\MPDAG$. To construct such an FT-MPDAG $\mathcal C$, we begin by constructing an FT-DAG $\ADMG$ as follows:  
(i) add the edge $X_t \to Y_t$ (or $X_t \leftarrow Y_t$, arbitrarily) 
(ii) for any nodes $S_{\mathbb W_1},S_{\mathbb W_2}$ of $\SCG$, if $S_{\mathbb W_1}\in Pa(S_{\mathbb W_2},\SCG)$, we add the edge $W^1_{t-\gamma_{\max}} \to W^2_t$ to $\ADMG$ (it implies in particular that $X_{t-\gamma_{\max}} \to X_t$, $Y_{t-\gamma_{\max}} \to Y_t$, $X_{t-\gamma_{\max}} \to Y_t$, and $Y_{t-\gamma_{\max}} \to X_t$ in $\ADMG$ since items 1 and 2 of the theorem holds; and that for every $S_{\mathbb Z} \in Pa(S_{\mathbb X},\SCG)$, $Z_{t-\gamma_{\max}} \to Y_t$ and $Z_{t-\gamma_{\max}}\to X_t$ in $\ADMG$ since item~3 of the theorem holds so $S_{\mathbb Z} \in Pa(S_{\mathbb Y},\SCG)$. By construction, it is immediate that the resulting FT-DAG $\ADMG$ is compatible with $\SCG$, and we finally define $\MPDAG$ as the FT-MPDAG obtained from $\ADMG$ and incorporating only the background knowledge given by the SCG. The edge $X_t - Y_t$ must remain unoriented in $\MPDAG$ because all triples involving $X_t$ and $Y_t$ are shielded by construction, so neither the UC-Rule nor Meek’s Rule 1 can be applied. 
Moreover, since no instantaneous edges connect to $X_t$ or $Y_t$, none of the other Meek rules can be applied (as no cycle can occur between present and past, and by construction there is only one instantaneous edge : $X_t-Y_t$). We have thus shown that if the items of the theorem hold, it is always possible to construct an FT-MPDAG $\MPDAG$ in which the edge between $X_t$ and $Y_t$ is unoriented $(\text{status}(X_t,Y_t;\ADMG)=1)$, that is $\text{status}(X_t,Y_t;\ADMG)$ is not $s$-orientable.
\end{proof}

Figure~\ref{fig:theorem1} illustrates the statement of the theorem. It shows that in any FT-MPDAG $\MPDAG$ compatible with $\SCG$, the edge between $Y_t$ and $Z_t$ is orientable due to the directed edge $S_\mathbb Y\to S_\mathbb Z$ in $\SCG$, while the edge between $Z_t$ and $W_t$ is orientable because of the unshielded collider $S_\mathbb W \rightleftarrows S_\mathbb Z \leftarrow S_\mathbb Y$ in $\SCG$. In contrast, the edge between $X_t$ and $Y_t$ is not $s$-orientable, since there exists an FT-MPDAG in which this edge is unoriented, as illustrated in (b).

We emphasize that the theorem above is \emph{sound and complete} in the following sense: if an edge between $X_t$ and $Y_t$ is $s$-orientable according to the theorem, then there cannot be an unoriented edge between $X_t$ and $Y_t$ in any FT-MPDAG  compatible with the SCG; conversely, if the edge is not $s$-orientable according to the theorem, there exists at least one FT-MPDAG compatible with the SCG where the edge between $X_t$ and $Y_t$ is unoriented. In other words, in these cases we cannot guarantee orientability prior to applying a causal discovery algorithm such as tPC. 
However, this does not imply that, in cases of non-orientability, applying the tPC algorithm will \emph{always} fail to orient the edge. It is possible that there exists a specific distribution for which the edge is orientable, and if tPC is applied to data generated from this distribution, the edge may in fact be correctly oriented. 
This is clarified by the following proposition.

\begin{proposition}
\label{prop:exists}
Let $\SCG$ be an SCG.
Under Assumptions~\ref{assumption:causal_sufficiency},~\ref{assumption:faithfulness} and~\ref{assumption:stationarity}, $\exists\ADMG\in \compatible{\SCG}$  \st $\text{status}(X_t,Y_t;\ADMG)$ is  orientable from $(\SCG, \mathbb{P})$ where $\mathbb{P}$ is the set of distributions compatible with $\ADMG$.
\end{proposition}

\begin{proof}
    Consider the trivial case without instantaneous relations which contains only directed edges oriented by time.
\end{proof}

Finally, we emphasize that non $s$-orientable cases are rare, constituting only a small fraction of all SCG configurations of a given size: among all possible SCGs of 5 nodes, less than 2\% are not fully $s$-orientable (details in supplementary material).

\begin{figure}[t]
    \centering
    \begin{subfigure}[b]{0.11\textwidth}
        \centering
        \begin{tikzpicture}[black, circle, draw, inner sep=0]
            \tikzset{nodes={draw,rounded corners},minimum height=0.7cm,minimum width=0.7cm}
            \node (A) at (0,0) {$S_\mathbb X$};
            \node (B) at (1.2,0) {$S_\mathbb Y$};
            \node (C) at (0,-1) {$S_\mathbb Z$};
            \node (D) at (1.2,-1) {$S_\mathbb W$};

            \draw[->,>=latex, color = blue] (A) edge[bend left=15] (B);
            \draw[->,>=latex, color = blue] (B) edge[bend left=15] (A);
            \draw[->,>=latex, color = purple] (C) edge[bend left=15] (D);
            \draw[->,>=latex, color = purple] (D) edge[bend left=15] (C);
            \draw[->,>=latex, color = orange] (B) -- (C);
            \draw[->,>=latex] (C) edge[loop below, looseness=1, min distance=5mm] (C);
            \draw[->,>=latex] (A) edge[loop above, looseness=1, min distance=5mm] (A);
            \draw[->,>=latex] (B) edge[loop above, looseness=1, min distance=5mm] (B);
            \draw[->,>=latex] (D) edge[loop below, looseness=1, min distance=5mm] (D);
        \end{tikzpicture}
        \caption{SCG $\SCG$.}
        \label{fig:scg}
    \end{subfigure}
    \hfill
    \begin{subfigure}[b]{0.17\textwidth}
        \centering
        \begin{tikzpicture}[black, circle, draw, inner sep=0]
            \tikzset{nodes={draw,rounded corners},minimum height=0.7cm,minimum width=0.7cm}

            \node[draw=none] at (-0.6,0) {...};
            \node[draw=none] at (1.6,0) {...};
        
            \node[draw=none] at (-0.6,-1) {...};
            \node[draw=none] at (1.6,-1) {...};
        
            \node[draw=none] at (-0.6,-2) {...};
            \node[draw=none] at (1.6,-2) {...};

            \node[draw=none] at (-0.6,-3) {...};
            \node[draw=none] at (1.6,-3) {...};
            
            \node (A-1) at (0,0) {$X_{t-1}$};
            \node (A) at (1,0) {$X{t}$};
            \node (B-1) at (0,-1) {$Y_{t-1}$};
            \node (B) at (1,-1) {$Y_{t}$};
            \node (C-1) at (0,-2) {$Z_{t-1}$};
            \node (C) at (1,-2) {$Z_{t}$};
            \node (D-1) at (0,-3) {$W_{t-1}$};
            \node (D) at (1,-3) {$W_{t}$};

            \draw[->,>=latex] (C-1) -- (C);
            \draw[->,>=latex] (A-1) -- (A);
            \draw[->,>=latex] (D-1) -- (D);
            \draw[->,>=latex] (B-1) -- (B);

            \draw[->,>=latex] (A-1) -- (B);
            \draw[->,>=latex] (B-1) -- (A);
            
            \draw[-,>=latex, color = blue] (B) -- (A);
            \draw[-,>=latex, color = blue] (B-1) -- (A-1);
            
            \draw[->,>=latex] (B-1) -- (C);
            \draw[->,>=latex, color = orange] (B) -- (C);
            \draw[->,>=latex, color = orange] (B-1) -- (C-1);
    
            \draw[->,>=latex] (D-1) -- (C);
            \draw[->,>=latex] (C-1) -- (D);
            \draw[->,>=latex, color = purple] (C) -- (D);
            \draw[->,>=latex, color = purple] (C-1) -- (D-1);
        \end{tikzpicture}
        \caption{FT-MPDAG $\MPDAG_1$.}
        \label{fig:leg1}
    \end{subfigure}
    \hfill
    \begin{subfigure}[b]{0.17\textwidth}
        \centering
        \begin{tikzpicture}[black, circle, draw, inner sep=0]
            \tikzset{nodes={draw,rounded corners},minimum height=0.7cm,minimum width=0.7cm}

            \node[draw=none] at (-0.6,0) {...};
            \node[draw=none] at (1.6,0) {...};
        
            \node[draw=none] at (-0.6,-1) {...};
            \node[draw=none] at (1.6,-1) {...};
        
            \node[draw=none] at (-0.6,-2) {...};
            \node[draw=none] at (1.6,-2) {...};

            \node[draw=none] at (-0.6,-3) {...};
            \node[draw=none] at (1.6,-3) {...};
            
            \node (A-1) at (0,0) {$X_{t-1}$};
            \node (A) at (1,0) {$X_{t}$};
            \node (B-1) at (0,-1) {$Y_{t-1}$};
            \node (B) at (1,-1) {$Y_{t}$};
            \node (C-1) at (0,-2) {$Z_{t-1}$};
            \node (C) at (1,-2) {$Z_{t}$};
            \node (D-1) at (0,-3) {$W_{t-1}$};
            \node (D) at (1,-3) {$W_{t}$};

            \draw[->,>=latex] (C-1) -- (C);
            \draw[->,>=latex] (A-1) -- (A);
            \draw[->,>=latex] (D-1) -- (D);
            \draw[->,>=latex] (B-1) -- (B);

            \draw[->,>=latex] (B-1) -- (A);
            
            \draw[->,>=latex, color = blue] (A) -- (B);

             \draw[->,>=latex, color = blue] (A-1) -- (B-1);
            
            \draw[->,>=latex] (B-1) -- (C);
            \draw[->,>=latex, color = orange] (B) -- (C);
            \draw[->,>=latex, color = orange] (B-1) -- (C-1);
    
            \draw[->,>=latex] (D-1) -- (C);
            \draw[->,>=latex] (C-1) -- (D);
            \draw[->,>=latex, color = purple] (C) -- (D);
            \draw[->,>=latex, color = purple] (C-1) -- (D-1);
        \end{tikzpicture}
        \caption{FT-MPDAG $\MPDAG_2$.}
        \label{fig:leg2}
    \end{subfigure}
    \caption{SCG $\SCG$ with two examples of FT-MPDAG illustrating Theorem~\ref{th:theorem1}. Orange and purple edges are $s$-orientable and thus oriented in the FT-MPDAGs, while blue edge is non-$s$-orientable and may remain unoriented as in (b).}
    \label{fig:theorem1}
\end{figure}

\section{Implications on Causal Effects Identification}
\label{sec:implications}

The previous section addressed the problem of identifying the orientation of an edge in an FT-MPDAG. This provides a qualitative understanding of the relationship between two variables, indicating whether $X_{t-\gamma}$ causes $Y_t$, whether the reverse holds, whether there is no direct causal relation between them, or whether the orientation cannot guaranteed to be determined from the SCG and a faithful distribution. In this section, we shift our focus to the quantitative implications of these results, namely how they inform the identification and estimation of causal effects.
SCMs are particularly useful because they allow for the quantification of causal effects from observational data. 
There are many types of causal effects; in this section, we focus on two central ones: the total effect~\citep{Pearl_1995} and the controlled direct effect~\citep{Pearl_2001}.
The \emph{total effect} of $X_{t-\gamma}$ on $Y_t$ captures the overall causal influence transmitted both directly and through mediating variables, whereas the \emph{direct effect} isolates the portion of this influence that is not mediated by other variables. Formally, the total effect on $Y_t$ of an intervention on $X_{t-\gamma}$ changing from $x_0$ to $x_1$ is defined as  
$\mathbb{E}[Y_t \mid \interv{X_{t-\gamma}=x_1}]
- \mathbb{E}[Y_t \mid \interv{X_{t-\gamma}=x_0}]$,
and the controlled direct effect is defined as 
$\mathbb{E}[Y_t \mid \interv{X_{t-\gamma}=x_1}, \interv{\mathbbl{S}_{Y_t, X_{t-\gamma},\ADMG}=\mathbbl{s}_{Y_t, X_{t-\gamma},\ADMG}}]
- \mathbb{E}[Y_t \mid \interv{X_{t-\gamma}=x_0}, \interv{\mathbbl{S}_{Y_t, X_{t-\gamma},\ADMG}=\mathbbl{s}_{Y_t, X_{t-\gamma},\ADMG}}]$,
where  $\mathbbl{S}_{Y_t, X_{t-\gamma},\ADMG}=Pa(Y_t,\ADMG)\setminus \{X_t\}$ and $\interv{\cdot}$ denotes an intervention.
For brevity, we will denote the total effect and the controlled direct effect, respectively, as $\probac{y_t}{\interv{x_{t-\gamma}}}$ and $\probac{y_t}{\interv{x_{t-\gamma}}, \interv{\mathbbl{s}_{Y_t, X_{t-\gamma},\ADMG}}}$.

A causal effect is said to be \emph{identifiable} when it can be uniquely computed from any distribution compatible with an FT-DAG. 
In other words, identifying such a causal effect consists rewriting the causal effect using a do-free expression (only in terms of observed variables without any need to perform a real intervention).  
When the FT-DAG is specified and there are no unobserved confounders (Assumption~\ref{assumption:causal_sufficiency}), causal effects are always identifiable from the FT-DAG~\citep{Pearl_2000}.
In practice, however, it is often easier to obtain the SCG (for instance, from expert knowledge) than the full FT-DAG, since the SCG represents a higher level of abstraction that does not require knowledge of the precise causal relations between variables at each time step. From the SCG, causal effects can still be identified~\citep{Assaad_2024, Ferreira_2024}, but due to the presence of cycles, some causal effects may remain non-identifiable solely from the SCG. Moreover, the identification conditions in SCGs are intricate and typically require examining multiple paths.  

For these reasons, one may consider using causal discovery techniques to estimate the FT-DAG. Given that the SCG is available, it can naturally be incorporated as background knowledge in the discovery process. However, as discussed earlier, constraint-based procedures do not recover the FT-DAG directly, but rather a partially directed graph, the FT-MPDAG. Because the FT-MPDAG may contain unoriented edges, causal effects are not always identifiable from it. Nevertheless, in some cases it has been shown that causal effects are identifiable from the FT-MPDAG~\citep{Perkovic_2020, Flanagan_2020}.  
This raises the following question: given the SCG, can we guarantee that the causal effect will be identifiable once a causal discovery step is performed? The following propositions provide an answer to this question. 

\begin{proposition}[Identifiability of total effects using $s$-orientability]
\label{prop:totaleffect}
Let $\SCG$ be an SCG. Under Assumptions~\ref{assumption:causal_sufficiency},~\ref{assumption:faithfulness} and~\ref{assumption:stationarity}, if for every $S_\mathbb{Z}\in Ne(S_\mathbb{X}, \SCG)$, 
$\forall\ADMG\in \compatible{\SCG}$,
$\text{status}(Z_t,X_t;\ADMG)$
is $s$-orientable, then $P(Y_t\mid \interv{x_{t-\gamma}})$ is identifiable.
\end{proposition}


\begin{proof}
The result follows directly from the backdoor criterion~\citep{Pearl_2000}, since under Assumption~\ref{assumption:causal_sufficiency} the parents of the treatment form a valid adjustment set. Under the $s$-orientability condition, all edges connected to the treatment can be oriented, allowing recovery of its parents and identification of the total effect.
\end{proof}

\begin{proposition}[Identifiability of controlled direct effects using $s$-orientability]\label{prop:directeffect}
Let $\SCG$ be an SCG. Under Assumptions~\ref{assumption:causal_sufficiency},~\ref{assumption:faithfulness} and~\ref{assumption:stationarity}, if for every $S_\mathbb{Z}\in Ne(S_\mathbb{Y}, \SCG)$, 
$\forall\ADMG\in \compatible{\SCG}$,
$\text{status}(Z_t,Y_t;\ADMG)$ is $s$-orientable, then 
$\probac{y_t}{\interv{x_{t-\gamma}}, \interv{\mathbbl{s}_{Y_t, X_{t-\gamma},\ADMG}}}$
is identifiable. 
\end{proposition}


\begin{proof}
The result follows directly from Theorem~5.4 of~\citet{Flanagan_2020}, which states that a controlled direct effect is identifiable from the FT-MPDAG if and only if all edges adjacent to the outcome are oriented. Under the $s$-orientability condition, these edges can be oriented, allowing recovery of the relevant adjustment set and identification of the controlled direct effect.
\end{proof}

These results provide theoretical guarantees for causal effect identification, even in cases where identification theorems applied to the SCG alone~\citep{Assaad_2024,Ferreira_2024} do not succeed. Moreover, they offer simpler criteria that can be verified directly on the SCG, relying only on pairs of nodes rather than complex path-based conditions.



\section{Discussion}
\label{sec:discussion}
In this paper, we demonstrated that the background knowledge encoded in an SCG makes it possible to determine which edges are guaranteed to be oriented after applying a causal discovery algorithm to a faithful and causally sufficient distribution. At first glance, this may appear straightforward, for instance, an SCG without bidirected edges naturally enables the orientation of many edges that remain ambiguous in the output of causal discovery. However, the surprising and more insightful aspect of our results is that even bidirected edges in an SCG can provide valuable information. We showed that it is possible to guarantee the orientation of an edge at the micro level (e.g., between $X_t$ and $Y_t$) even in cases where $S_{\mathbb{X}} \leftrightarrows S_{\mathbb{Y}}$ is in the SCG, provided that at least one of these macro vertices does not have a self-loop, or there exists a parent of $S_\mathbb X$ that is not a parent of $S_\mathbb Y$. Moreover, we established that these insights allow us to simplify existing identifiability conditions in the literature for quantifying causal effects, by leveraging SCGs in combination with access to a faithful distribution compatible with the SCG. 

In this work, we assumed that the SCG is provided as background knowledge (e.g., by an expert). However, we note that it may also be obtained using an SCG discovery algorithm~\citep{Assaad_Entropy_2022,Wahl_2023,Wahl_2024,Ninad_2025}. For instance, suppose that a coarse-grained SCG has already been learned beforehand, and that we now seek to perform a fine-grained causal discovery of the FT-MPDAG. In this case, the previously obtained SCG can be used in the same way as expert knowledge to assess the a priori orientability guarantees established by our results. The key point is that the procedure used to obtain an SCG does not affect our results, as long as the provided SCG is considered to represent the true underlying dynamic SCM.


A key limitation of this work concerns the strong assumptions generally required by the PC algorithm in causal discovery. Moreover, the stationarity assumption (Assumption~\ref{assumption:stationarity}) is crucial for our results, as it ensures (together with the definition of SCGs) the existence of certain edges in the FT-DAG that follow from the SCG. For instance, if $S_\mathbb{X} \rightleftarrows S_\mathbb{Y}$ is in the SCG, we assume that there necessarily exists an edge from some past or present instance of $\{Y_{t_0},\cdots,Y_{t_{\max}}\}$ pointing to $X_t$ and an edge from $\{X_{t_0},\cdots,X_{t_{\max}}\}$ pointing to $Y_t$. Moreover, stationarity guarantees that orientations remain consistent when shifting the time lag, thereby ruling out situations where, for example, $X_{t-1} \to Y_{t-1}$ and $X_t \leftarrow Y_t$ coexist in the same FT-DAG, which would invalidate parts of our proofs. While this assumption is necessary for our theoretical results, it may appear restrictive in applications. It should be noted, however, that this is a standard assumption adopted by most causal discovery algorithms in time series settings~\citep{Peters_2013,Malinsky_2018,Runge_2019,Runge_2020,Assaad_2022}. In addition, a weaker version can be considered, where stationarity is assumed only on a subinterval $\mathcal{T} \subset [t_0, t_{\max}]$ rather than on the entire interval. This is feasible provided that multiple repetitions of the time series are available to perform causal discovery. Otherwise, full stationarity must be assumed, in which case the series can be partitioned into subsamples and causal discovery can be carried out on these subsampled segments. 

Another strong assumption arises in the definition of the SCG: we assume that an edge $S_\mathbb{X} \to S_\mathbb{Y}$ exists in the SCG if and only if there exist times $t'$ and $t$ \st $X_{t'} \to Y_t$ in the FT-DAG. Since our proofs rely on this assumption, we do not allow cases where an effect is posited in the SCG but absent from the FT-DAG. Consider a concrete example: suppose that, a priori, there is an effect from some instance of $\{X_{t_0},\cdots,X_{t_{\max}}\}$ to some instance of $\{Y_{t_0},\cdots,Y_{t_{\max}}\}$, and that there also exists another instance where $\{Y_{t_0},\cdots,Y_{t_{\max}}\}$ affects $\{X_{t_0},\cdots,X_{t_{\max}}\}$, together with a self-loop on $S_\mathbb{Y}$. Such prior knowledge translates into the SCG depicted in Figure~\ref{fig:lemma2_visual_combined}(d). Now, suppose that in reality the edge $S_\mathbb{X} \to S_\mathbb{Y}$ does not exist (i.e., the prior was incorrect), and instead we actually have $Y_t \to X_t$. In this case, causal discovery would identify the presence of an edge between $X_t$ and $Y_t$, and by incorporating the prior knowledge, we would incorrectly orient it as $X_t \to Y_t$. Our results are therefore strong in the sense that they rely on having reliable prior knowledge encoded in the SCG, faithful to its definition. This highlights the importance of critically assessing the validity of SCG edges before reasoning with the results presented in this paper.


For future work, it would be interesting to extend these results to the FCI algorithm~\citep{Spirtes_2000}, an extension of the PC algorithm that does not rely on causal sufficiency.
Moreover, while Propositions~\ref{prop:totaleffect} and \ref{prop:directeffect} provide sufficient conditions for identifying total and controlled direct effects, respectively, their completeness remains an open question. 
Investigating this aspect represents another promising direction for future research.


\subsection*{Acknowledgements}
This work was supported by the CIPHOD project (ANR-23-CPJ1-0212-01) and by funding from the French government managed by the National Research Agency (ANR) under the France 2030 program (ANR-23-IACL-0007).



%% file: appendix.tex
\section*{Empirical scarcity of non $s$-orientable cases}
\label{sec:scarcity}

The structure of a non $s$-orientable pair in an SCG satisfying the assumptions is highly specific according to Theorem~\ref{th:theorem1}, and can therefore be expected to be rare. To gain an idea of its frequency, we enumerate all possible SCGs for a given number of nodes and check whether there exists at least one pair of nodes in the SCG that is non $s$-orientable (i.e., a pair of nodes satisfying Theorem~\ref{th:theorem1}). Since the number of possible SCGs with $n$ nodes is huge $(2^{n\times n})$, we only report results up to $n=5$. The results are summarized in Table~\ref{tab:scg-enum}.

\begin{table}[h!]
\centering
\begin{tabular}{c|r|r|r}
\hline
$n$ & \# SCGs & \# Not fully $s$-orientable & \% \\
\hline
2 & 16 & 1 & 6.25 \\
3 & 512 & 28 & 5.47 \\
4 & 65,536 & 2,256 & 3.44 \\
5 & 33,554,432 & 613,616 & 1.83 \\
\hline
\end{tabular}
\caption{Enumeration of all SCGs up to $n=5$ nodes, showing the number of graphs that are not fully $s$-orientable.}
\label{tab:scg-enum}
\end{table}

These results already illustrate the rarity of such cases and show that, in the vast majority of theoretically possible SCGs, all edges are $s$-orientable using the background knowledge induced by the SCG. Moreover, we emphasize that these rare cases are non-$s$-identifiable from the SCG, but this does not imply that a causal discovery algorithm will fail to return an FT-DAG. Indeed, there is no a priori guarantee when considering only the SCG; however, in some cases (which may be numerous), it is still possible for an edge to be oriented in our specific setting. This is illustrated in Figure~\ref{fig:theorem1}(c), where we observe a case in which a non $s$-orientable edge is nonetheless oriented. 

Note, however, that while this observation suggests that non–fully $s$-orientable SCGs are relatively rare in a density sense over the space of all SCGs, one should keep in mind that such graphs may occur more frequently in practice, since real systems do not arise from a uniform sampling process. For instance, our computation includes SCGs with disconnected components, which are unlikely in realistic settings. This observation should therefore be understood only as a general indication rather than a practical estimate.